\documentclass{article}

\usepackage{amsmath,amsfonts,bm}

\def\eqref#1{equation~\ref{#1}}

\def\1{\bm{1}}

\usepackage{geometry}

\newgeometry{vmargin={15mm}, hmargin={35mm,35mm}}   

\usepackage[utf8]{inputenc} 
\usepackage[T1]{fontenc}    
\usepackage{hyperref}       
\usepackage{url}            
\usepackage{booktabs}       
\usepackage{amsfonts}       
\usepackage{nicefrac}       
\usepackage{microtype}      

\usepackage{url}
\usepackage{amsmath,amsthm,dsfont}
\usepackage{color}
\usepackage{graphicx}
\usepackage{ulem}
\usepackage{accents}
\usepackage{soul}

\newtheorem{theorem}{Theorem}

\newtheorem{remark}{Remark}

\usepackage{algorithm}
\usepackage[noend]{algpseudocode}

\usepackage{booktabs}
\usepackage{threeparttable}

\title{ResNets Ensemble via the Feynman-Kac Formalism to Improve Natural and Robust Accuracies}

\author{
Bao Wang \\
  Department of Mathematics\\
  University of California, Los Angeles\\
  \texttt{wangbaonj@gmail.com}\\
  \and
  Bingjie Yuan \\Computer Science Department\\
  Tsinghua University\\
   \texttt{ybj14@mails.tsinghua.edu.cn} \\
  \and
  Zuoqiang Shi\\
  Yau Mathematical Sciences Center \\
  Tsinghua University\\
  \texttt{zqshi@tsinghua.edu.cn} \\
  \and 
  Stanley J. Osher\\
  Department of Mathematics\\
  University of California, Los Angeles\\
  \texttt{sjo@math.ucla.edu}\\
}

\begin{document}

\maketitle

\begin{abstract}
Empirical adversarial risk minimization (EARM) is a widely used mathematical framework to robustly train deep neural nets (DNNs) that are resistant to adversarial attacks. However, both natural and robust accuracies, in classifying clean and adversarial images, respectively, of the trained robust models are far from satisfactory. In this work, we unify the theory of optimal control of transport equations with the practice of training and testing of ResNets. Based on this unified viewpoint, we propose a simple yet effective ResNets ensemble algorithm to boost the accuracy of the robustly trained model on both clean and adversarial images. The proposed algorithm consists of two components: First, we modify the base ResNets by injecting a variance specified Gaussian noise to the output of each residual mapping. Second, we average over the production of multiple jointly trained modified ResNets to get the final prediction. These two steps give an approximation to the Feynman-Kac formula for representing the solution of a transport equation with viscosity, or a convection-diffusion equation. For the CIFAR10 benchmark, this simple algorithm leads to a robust model with a natural accuracy of {\bf 85.62}\% on clean images and a robust accuracy of ${\bf 57.94 \%}$ under the 20 iterations of the IFGSM attack, which outperforms the current state-of-the-art in defending against IFGSM attack on the CIFAR10. Both natural and robust accuracies of the proposed ResNets ensemble can be improved dynamically as the building block ResNet advances. The code is available at: \url{https://github.com/BaoWangMath/EnResNet}.
\end{abstract}

\section{Introduction}
Deep learning (DL) achieves great success in image and speech perception \cite{DeepLearningReview:2015}. Residual learning revolutionizes the deep neural nets (DNNs) architecture design and makes training of the ultra-deep, up to more than one thousand layers, DNNs practical \cite{ResNet}. The idea of residual learning motivates the development of a good number of related powerful DNNs, e.g., Pre-activated ResNet \cite{PreActResNet}, ResNeXt \cite{ResNext}, DenseNet \cite{DenseNet}, and many others. Neural nets ensemble is a learning paradigm where many DNNs are jointly used to improve the performance of individual DNNs \cite{Hansen:1990}.
\medskip

Despite the extraordinary success of DNNs in image and speech recognition, their vulnerability to adversarial attacks raises concerns when applying them to security-critical tasks, e.g., autonomous cars \cite{Akhtar:2018,Attack:Tesla}, robotics \cite{Giusti:2016Drones}, and DNN-based malware detection systems \cite{PapernotSecurity:2016,PapernotMalware:2016}.
Since the seminal work of Szegedy et al. \cite{szegedy2013intriguing}, recent research shows that DNNs are vulnerable to many kinds of adversarial attacks including physical, poisoning, and inference attacks \cite{Chen:2017,CWAttack:2016,PapernotAttack:2016,Goodfellow:2014AdversarialTraining,Ilyas:2018,Athalye:2018B,Athalye:2018}. The physical attacks occur during the data acquisition, the poisoning and inference attacks happen during the training and testing phases of machine learning (ML), respectively.
\medskip

The adversarial attacks have been successful in both white-box and black-box scenarios. In white-box attacks, the adversarial attacks have access to the architecture and weights of the DNNs. In black-box attacks, the attacks have no access to the details of the underlying model. Black-box attacks are successful because one can perturb an image to cause its misclassification on one DNN, and the same perturbed image also has a significant chance to be misclassified by another DNN; this is known as transferability of adversarial examples  \cite{DBLP:journals/corr/PapernotMG16}. Due to this transferability, it is straightforward to attack DNNs in a black-box fashion \cite{LiuYanpei:2016,Brendel:2017}. There exist universal perturbations that can imperceptibly perturb any image and cause misclassification for any given network \cite{Moosavi-Dezfooli_2017_CVPR}. Dou et al. \cite{Dou:2018}, analyzed the efficiency of many adversarial attacks for a large variety of DNNs. Recently, there has been much work on defending against these universal perturbations \cite{Akhtar_2018_CVPR}.
\medskip

The empirical adversarial risk minimization (EARM) is one of the most successful mathematical frameworks for certified adversarial defense. Under the EARM framework, adversarial defense for
$\ell_\infty$ norm based inference attacks can be formulated as solving the following EARM \cite{Madry:2018,Yin:2019-Rademacher}
\begin{equation}
\label{Adversarial-ERM}
\min_{f\in \mathcal{H}} \frac{1}{n} \sum_{i=1}^n \max_{\|\mathbf{x}'_i-\mathbf{x}_i\|_\infty \leq \epsilon} L(f(\mathbf{x}_i', \mathbf{w}), y_i),
\end{equation}
where $f(\cdot, \mathbf{w})$ is a function in the hypothesis class $\mathcal{H}$, e.g., ResNets, parameterized by $\mathbf{w}$. Here, $\{(\mathbf{x}_i, y_i)\}_{i=1}^n$ are $n$ i.i.d. data-label pairs drawn from some high dimensional unknown distribution $\mathcal{D}$, $L(f(\mathbf{x}_i, \mathbf{w}), y_i)$ is the loss associated with $f$ on the data-label pair $(\mathbf{x}_i, y_i)$. For classification, $L$ is typically selected to be the cross-entropy loss; for regression, the root mean square error is commonly used. The adversarial defense for other measure based attacks can be formulated similarly. As a comparison, empirical risk minimization (ERM) is used to train models in a natural fashion that generalize well on the clean data, where ERM is to solve the following optimization problem
\begin{equation}
\label{ERM}
\min_{f\in \mathcal{H}} \frac{1}{n} \sum_{i=1}^n L(f(\mathbf{x}_i, \mathbf{w}), y_i).
\end{equation}
\medskip

Many of the existing works try to defend against the inference attacks by finding a good approximation to the loss function in EARM. Project gradient descent (PGD) adversarial training is a representative work along this side that approximate EARM by replacing $\mathbf{x}_i'$ with the adversarial data that obtained by applying the PGD attack to the clean data \cite{Goodfellow:2014AdversarialTraining,Madry:2018,Na:2018}. Zhang et al.~\cite{Zhang:2019-Trades} replace the empirical adversarial risk by a linear combination of empirical and empirical adversarial risks. Besides finding a good surrogate to approximate the empirical adversarial risk, under the EARM framework, we can also improve the hypothesis class to improve the adversarial robustness of the trained robust models.
\medskip

\subsection{Our Contribution}
The robustly trained DNNs usually more resistant to adversarial attacks, however, they are much less accurate on clean images than the naturally trained models. A natural question is

\emph{Can we improve both natural and robust accuracies of the robustly trained DNNs?}
\medskip

In this work, we unify the training and testing of ResNets with the theory of transport equations (TEs). This unified viewpoint enables us to interpret the adversarial vulnerability of ResNets as the irregularity, which will be defined later, of the TE's solution. Based on this observation, we propose a new ResNets ensemble algorithm based on the Feynman-Kac formula. In a nutshell, the proposed algorithm consists of two essential components. First, for each $l = 1, 2, \cdots, M$ with $M$ being the number of residual mappings in the ResNet, we modify the $l$-th residual mapping from $\mathbf{x}_{l+1} = \mathbf{x}_l+\mathcal{F}(\mathbf{x}_l)$ (Fig.~\ref{fig:Structure} (a)) to $\mathbf{x}_{l+1} = \mathbf{x}_l+\mathcal{F}(\mathbf{x}_l) + N(0, \sigma^2\mathbf{I})$ (Fig.~\ref{fig:Structure} (b)), where $\mathbf{x}_l$ is the input, $\mathcal{F}$ is the residual mapping and $N(0, \sigma^2\mathbf{I})$ is Gaussian noise with a specially designed variance $\sigma^2$. Second, we average over multiple jointly and robustly trained
modified ResNets' outputs to get the final prediction (Fig.~\ref{fig:NeuralNets-EnResNet}). This ensemble algorithm improves the base model's accuracy on both clean and adversarial data.
The advantages of the proposed algorithm are summarized as follows:
\begin{itemize}
    \item It outperforms the current state-of-the-art in defending against inference attacks.
    
    \item It improves the natural accuracy of the adversarially trained models.
    
    \item Its defense capability can be improved dynamically as the base ResNet advances.
    
    \item It enables to train and integrate an ultra-large DNN for adversarial defense with a limited GPU memory.
    
    
    \item It is motivated from partial differential equation (PDE) theory, which introduces a new way to defend against adversarial attacks, and it is a complement to many other existing adversarial defenses.
\end{itemize}

\begin{figure}[h]
\centering
\begin{tabular}{cc}
\includegraphics[width=0.38\columnwidth]{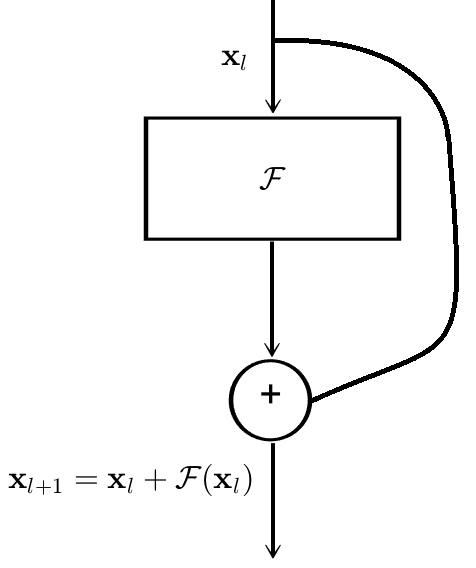}&
\includegraphics[width=0.50\columnwidth]{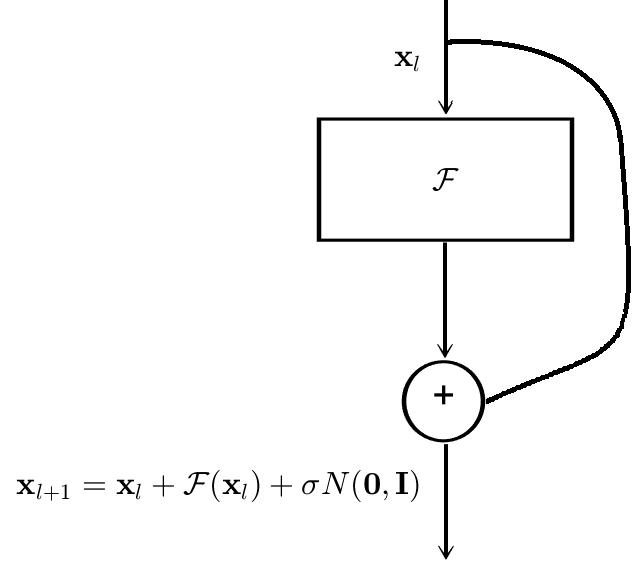}\\
(a) & (b)\\
\end{tabular}
\caption{(a) Residual mapping of the ResNet. (b) Gaussian noise injected residual mapping with $\sigma$ being the variance.}
\label{fig:Structure}
\end{figure}

\begin{figure}[h]
\centering
\begin{tabular}{c}
\includegraphics[width=0.75\columnwidth]{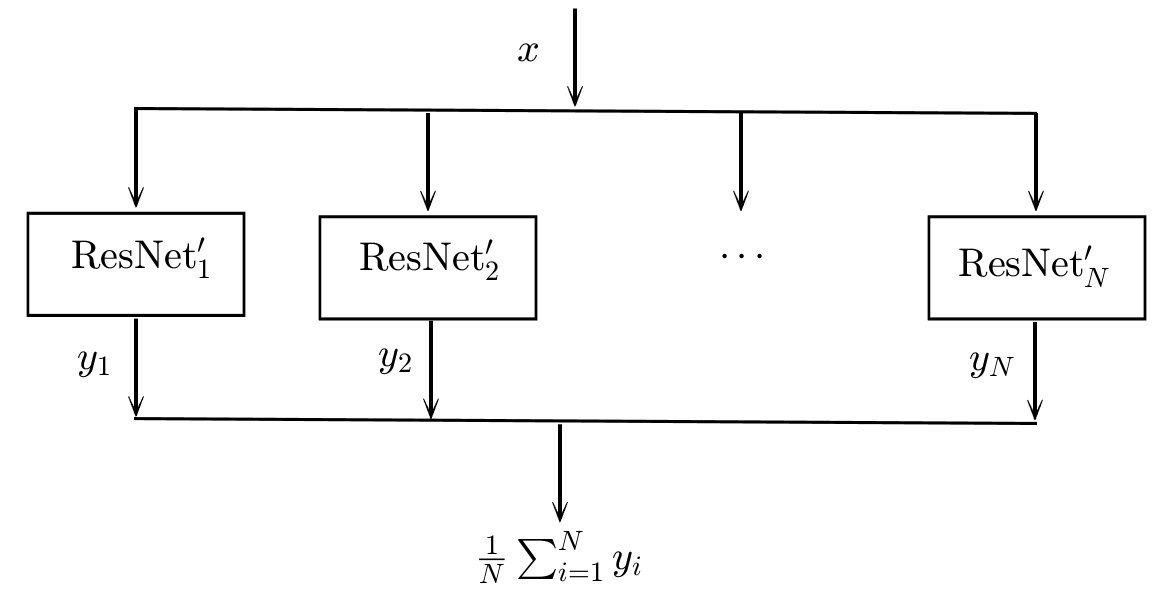}\\
\end{tabular}
\caption{Architecture of the EnResNet.}
\label{fig:NeuralNets-EnResNet}
\end{figure}

\subsection{Related Work}\label{Subsection-RelatedWork}
There is a massive volume of research over the last several years on defending against adversarial attacks for DNNs. Randomized smoothing transforms an arbitrary classifier $f$ into a "smoothed" surrogate classifier $g$ and is certifiably robust in $\ell_2$ norm based adversarial attacks \cite{Li:2018Smooth,Lecuyer:2019,CRK19,Wong:2018ICML,Cao:2017}. Among the randomized smoothing, one of the most popular ideas is to inject Gaussian noise to the input image and the classification result is based on the probability of the noisy image in the decision region. Our adversarial defense algorithm injects noise into each residual mapping instead of the input image, which is different from randomized smoothing.
\medskip 

Robust optimization for solving EARM achieves great success in defending against inference attacks \cite{Madry:2018,Raghunathan:2018a,Raghunathan:2018b,Wong:2018b,Salman:2019}. Regularization in EARM can further boost the robustness of the adversarially trained models \cite{Yin:2019-Rademacher,Kurakin:2017,Ross:2017,Zheng:2017}. The adversarial defense algorithms should learn a classifier with high test accuracy on both clean and adversarial data. To achieve this goal, Zhang et al. \cite{Zhang:2019-Trades} developed a new loss function, TRADES, that explicitly trades off between natural and robust generalization. To the best our of knowledge, TRADES is the current state-of-the-art in defending against inference attacks on the CIFAR10. Throughout this paper, we regard TRADES as the benchmark.
\medskip

Modeling DNNs as ordinary differential equations (ODEs) has drawn lots of attention recently. Chen et al. proposed neural ODEs for DL \cite{NODE}. E \cite{E:2017} modeled training ResNets as solving an ODE optimal control problem. Haber and Ruthotto \cite{Haber:2017} constructed stable DNN architectures based on the properties of numerical ODEs. Lu, Zhu and et al. \cite{BinDong:2018,Zhu:2018} constructed novel architectures for DNNs, which were motivated from the numerical discretization schemes for ODEs. Sun et al. \cite{SunDu:2018} modeled training of ResNets as solving a stochastic differential equation.
\medskip 

Model averaging with multiple stochastically trained identical DNNs is the most straightforward ensemble technique to improve the predictive power of base DNNs. This simple averaging method has been a success in image classification for ILSVRC competitions. Different groups of researchers use model averaging for different base DNNs and won different ILSVRC competitions \cite{AlexNet:2012,VGG:2014,ResNet}.
This widely used unweighted averaging ensemble, however, is not data-adaptive and is sensitive to the presence of excessively biased base learners. Ju et al., recently investigated ensemble of DNNs by many different ensemble methods, including unweighted averaging, majority voting, the Bayes Optimal Classifier, and the (discrete) Super Learner, for image recognition tasks. They concluded that the Super Learner achieves the best performance among all the studied ensemble algorithms \cite{Ju:2017}.
\medskip

Our work distinguishes from the existing work on DNN ensemble and feature and input smoothing from two major points: First, we inject Gaussian noise to each residual mapping in the ResNet. Second, we jointly train each component of the ensemble instead of using a sequential training.

\subsection{Organization}
We organize this paper in the following way: In section~\ref{Section-Motivation}, we model the ResNet as a TE and give an explanation for ResNet's adversarial vulnerability. In section~\ref{Section-Algorithms}, we present a new ResNet ensemble algorithm that motivated from the Feynman-Kac formula for adversarial defense. In section~\ref{Section-Results}, we present the natural accuracy of the EnResNets and their robust accuracy under both white-box and blind PGD and C\&W attacks, and compare with the current state-of-the-art. In section~\ref{Section-Generalization}, we generalize the algorithm to ensemble of different neural nets and numerically verify its efficacy. Our paper ends up with some concluding remarks.

\section{Theoretical Motivation and Guarantees}\label{Section-Motivation}
\subsection{Transport Equation Modeling of ResNets}
The connection between training ResNet and solving optimal control problems of the TE is investigated in \cite{Wang:2018AdversarialDefense,BaoWang:2018NIPS,Li:2017PDE}. In this section, we derive the TE model for ResNet and explain its adversarial vulnerability from a PDE viewpoint. The TE model enables us to understand the data flow of the entire training and testing data in both forward and backward propagation in training and testing of ResNets; whereas, the ODE models focus on the dynamics of individual data points \cite{NODE}.
\medskip

As shown in Fig.~\ref{fig:Structure} (a), residual mapping adds a  skip connection to connect the input and output of the original mapping ($\mathcal{F}$), and the $l$-th residual mapping can be written as
$$\mathbf{x}_{l+1} = \mathcal{F}(\mathbf{x}_l, \mathbf{w}_l) + \mathbf{x}_l,$$
with $\mathbf{x}_0=\hat{\mathbf{x}} \in T\subset \mathbb{R}^d$ being a data point in the set $T$, $\mathbf{x}_l$ and $\mathbf{x}_{l+1}$ are the input and output tensors of the residual mapping. The parameters $\mathbf{w}_l$ can be learned by back-propagating the training error.
For $\forall\ \hat{\mathbf{x}} \in T$ with label $y$, the forward propagation of ResNet can be written as
\begin{eqnarray}
\label{PDE-Eq1}
\begin{cases}
\mathbf{x}_{l+1} = \mathbf{x}_l + \mathcal{F}(\mathbf{x}_l, \mathbf{w}_l), \ \ l = 0, 1, \dots, L-1, \ \ \mbox{with}\ \ \ \mathbf{x}_0 = \hat{\mathbf{x}},\\
\hat{y} \doteq f(\mathbf{x}_L),
\end{cases}
\end{eqnarray}
where $\hat{y}$ is the predicted label, $L$ is the number of layers, and
$f(\mathbf{x}) = {\rm softmax}(\mathbf{w}_0\cdot \mathbf{x})$ be the output activation
with $\mathbf{w}_0$ being the trainable parameters. For the widely used residual mapping in the pre-activated ResNet \cite{PreActResNet}, as shown in Fig.~\ref{Res-Block-detail} (a), we have
\begin{eqnarray}
\label{velocity}
\mathcal{F}(\mathbf{x}_l, \mathbf{w}_l) = \mathbf{w}^{\rm C2}_l\otimes
\sigma (\mathbf{w}^{\rm B2}_l \odot \mathbf{w}^{\rm C1}_l \otimes \sigma(\mathbf{w}^{\rm B1}_l \odot x_l)), 
\end{eqnarray}
where $\mathbf{w}_l^{\rm C1} (\mathbf{w}_l^{\rm B1})$ and $\mathbf{w}_l^{\rm C2} (\mathbf{w}_l^{\rm B2})$ are the first and second convolutional (batch normalization) layers of the $l$-th residual mapping, respectively, from top to bottom order. $\otimes$ and $\odot$ are the convolutional and batch normalization operators, respectively.
\medskip

\begin{figure}
\centering
\begin{tabular}{cc}
\includegraphics[width=0.11\columnwidth]{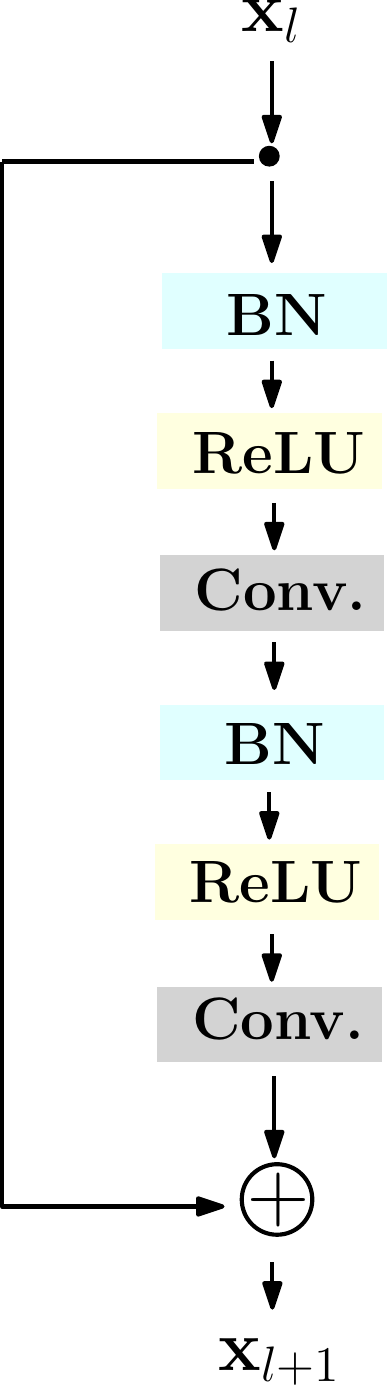}&
\includegraphics[width=0.65\columnwidth]{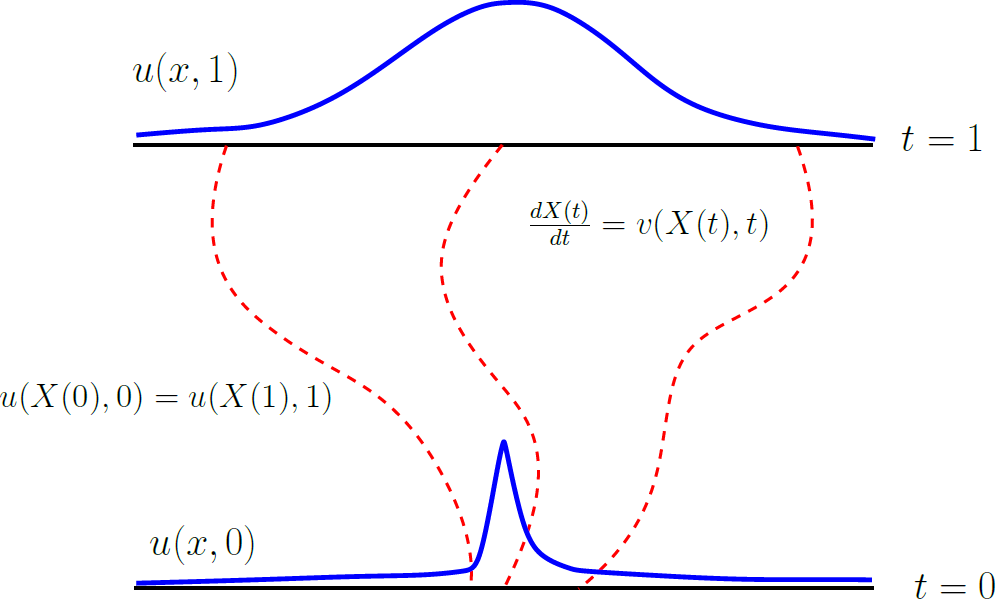}\\
(a) & (b) \\
\end{tabular}
\caption{(a) A detailed structure of the residual mapping in the pre-activated ResNet. (b) Demonstration of characteristic curves of the transport equation.}
\label{Res-Block-detail}
\end{figure}

Next, we introduce a temporal partition: 
let $t_l=l/L$, for $l=0, 1, \cdots, L$, with the time interval $\Delta t = 1/L$. Without considering dimensional consistency, we regard $\mathbf{x}_l$ in Eq.~(\ref{PDE-Eq1}) as the value of $\mathbf{x}(t)$ at the time slot $t_l$, 
so Eq.~(\ref{PDE-Eq1}) can be rewritten as
\begin{eqnarray}
\label{PDE-Eq2}
\begin{cases}
\mathbf{x}(t_{l+1}) = \mathbf{x}(t_l) + \Delta t\cdot\overline{F}(\mathbf{x}(t_l), \mathbf{w}(t_l)),  \ \ \ l = 0, 1, \dots, L-1,\ \ \mbox{with}\ \ \mathbf{x}(0) = \hat{\mathbf{x}} \\
\hat{y} \doteq f(\mathbf{x}(1)),
\end{cases}
\end{eqnarray}
where $\overline{F} \doteq \frac{1}{\Delta t}\mathcal{F}$. Eq.~(\ref{PDE-Eq2}) is the forward Euler discretization of the following ODE
\begin{equation}
\label{eq:character}
\frac{d\mathbf{x}(t)}{dt} = \overline{F}(\mathbf{x}(t), \mathbf{w}(t)), \ \ \mathbf{x}(0) = \hat{\mathbf{x}}.
\end{equation}
Let $u(\mathbf{x}, t)$ be a function that is constant along the trajectory defined by Eq.~(\ref{eq:character}), as demonstrated in Fig.~\ref{Res-Block-detail} (b), then $u(\mathbf{x}, t)$ satisfies the following TE
\begin{equation}
\label{PDE-Eq4}
\frac{d}{dt}\left(u(\mathbf{x}(t), t)\right) = \frac{\partial u}{\partial t}(\mathbf{x}, t) + \overline{F}(\mathbf{x}, \mathbf{w}(t))\cdot \nabla u (\mathbf{x}, t) = 0,\ \ \mathbf{x}\in \mathbb{R}^d,
\end{equation}
the first equality is because of the chain rule and the second equality dues to the fact that $u$ is constant along the curve defined by Eq.~(\ref{eq:character}).
\medskip

If we enforce the terminal condition at $t=1$ for Eq.~(\ref{PDE-Eq4}) to be
$$
u(\mathbf{x}, 1) = {\rm softmax} (\mathbf{w}_0\cdot \mathbf{x}) := f(\mathbf{x}),
$$
then according to the fact that $u(\mathbf{x}, t)$ is constant along the curve defined by Eq.~(\ref{eq:character}) (which is called the characteristic curve for the TE defined in Eq.~(\ref{PDE-Eq4})), we have $u(\hat{\mathbf{x}}, 0) = u(\mathbf{x}(1), 1) = f(\mathbf{x}(1))$; therefore, the forward propagation of ResNet for $\hat{\mathbf{x}}$ can be modeled as computing $u(\hat{\mathbf{x}}, 0)$ along the characteristic curve of the following TE
\begin{eqnarray}
\label{PDE-Eq6}
\begin{cases}
\frac{\partial u}{\partial t}(\mathbf{x}, t) + \overline{F}(\mathbf{x}, \mathbf{w}(t))\cdot \nabla u(\mathbf{x}, t) = 0,\ \ \mathbf{x}\in \mathbb{R}^d,\\
 u(\mathbf{x}, 1) = f(\mathbf{x}).
\end{cases}
\end{eqnarray}
\medskip

Meanwhile, the backpropagation in training ResNets can be modeled as finding the velocity field, $\overline{F}(\mathbf{x}(t), \mathbf{w}(t))$, for the following control problem
\begin{eqnarray}
\label{PDE-Eq7}
\begin{cases}
\frac{\partial u}{\partial t} (\mathbf{x}, t)+ \overline{F}(\mathbf{x}, \mathbf{w}(t))\cdot \nabla u (\mathbf{x}, t)= 0,\ \ \mathbf{x}\in \mathbb{R}^d,\\
u(\mathbf{x}, 1) = f(\mathbf{x}),\ \ \mathbf{x} \in \mathbb{R}^d,\\
u(\mathbf{x}_i, 0) = y_i, \ \ \mathbf{x}_i\in T,\ \ \mbox{with}\ T,
\end{cases}
\end{eqnarray}
where $T$ is the training set that enforces the initial condition on the training data for the TE.
Note that in the above TE formulation of ResNet, $u(\mathbf{x}, 0)$ serves as the classifier 
and the velocity field $\overline{F}(\mathbf{x}, \mathbf{w}(t))$ encodes ResNet's architecture and weights. When $\overline{F}$ is very complex, $u(\mathbf{x}, 0)$ might be highly irregular i.e. a small change in the input $\mathbf{x}$ can lead to a massive change in the value of $u(\mathbf{x}, 0)$. This irregular function may have a good generalizability on clean images, but it is not robust to adversarial attacks. Fig.~\ref{fig-Solution-Advection-Diffusion} (a) shows a 2D illustration of $u(\mathbf{x}, 0)$ with the terminal condition $u(\mathbf{x}, 1)$ shown in Fig.~\ref{fig-Solution-Advection-Diffusion} (d);
we will discuss this in detail later in this section.
\begin{figure}[!h]
\centering
\begin{tabular}{cc}
\includegraphics[width=0.42\columnwidth]{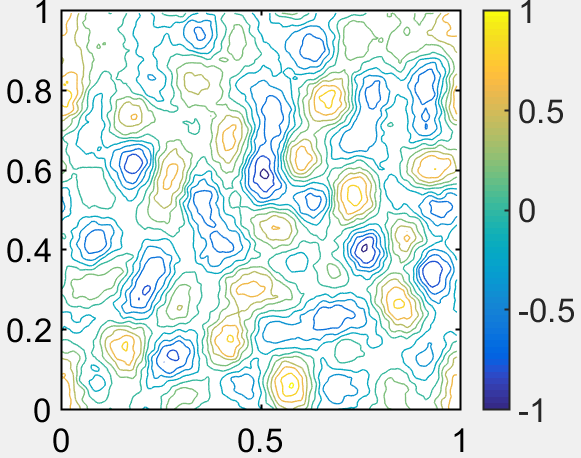}&
\includegraphics[width=0.42\columnwidth]{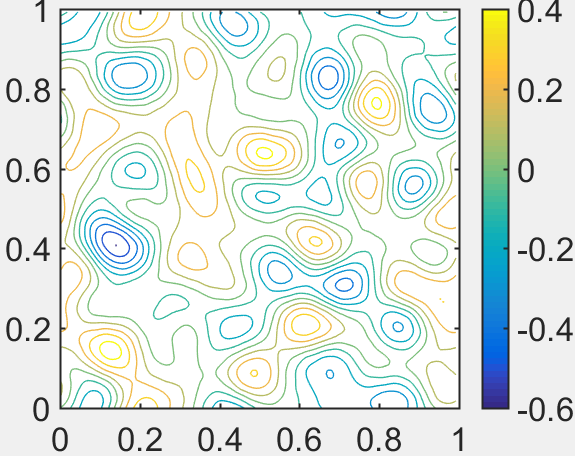}\\
(a) $\sigma=0$ & (b) $\sigma=0.01$ \\
\includegraphics[width=0.42\columnwidth]{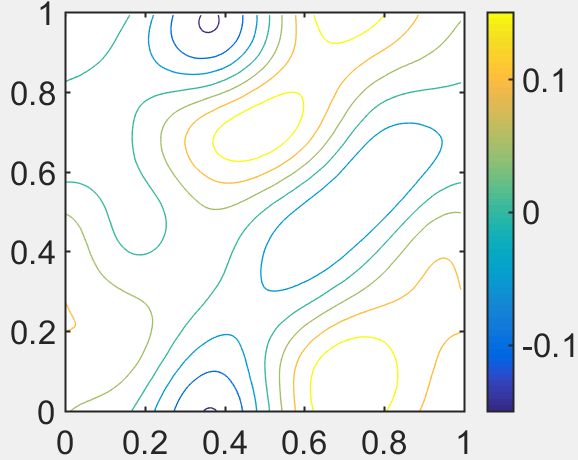}&
\includegraphics[width=0.42\columnwidth]{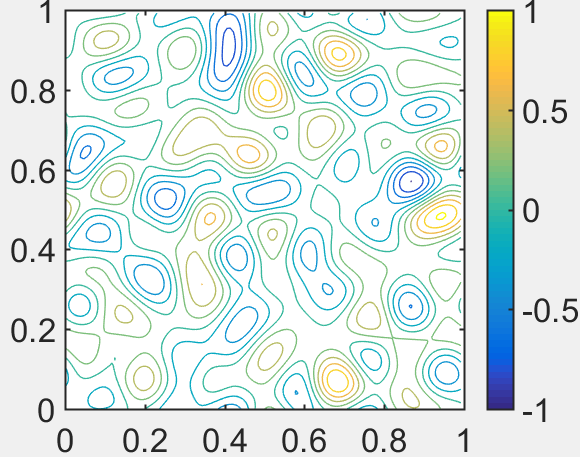}\\
(c) $\sigma=0.1$ & (d) $u(\mathbf{x}, 1)$\\
\end{tabular}
\caption{(d): terminal condition for Eq.~(\ref{PDE-Eq8}); (a), (b), and (c): solutions of the convection-diffusion equation, Eq.~(\ref{PDE-Eq8}), at $t=0$ with different diffusion coefficients $\sigma$.}
\label{fig-Solution-Advection-Diffusion}
\end{figure}

\subsection{Improving Robustness via Diffusion}
Using a specific level set of $u(\mathbf{x}, 0)$ in Fig.~\ref{fig-Solution-Advection-Diffusion} (a) for classification suffers from adversarial vulnerability: A tiny perturbation in $\mathbf{x}$ will lead the output to go across the level set, thus leading to misclassification. To mitigate this issue, we introduce a diffusion term $\frac{1}{2}\sigma^2\Delta u$ to Eq.~(\ref{PDE-Eq6}), with $\sigma$ being the diffusion coefficient and 
$$
\Delta = \frac{\partial^2}{\partial x_1^2} + \frac{\partial^2}{\partial x_2^2} +  \cdots + \frac{\partial^2}{\partial x_d^2},
$$
is the Laplace operator in $\mathbb{R}^d$. The newly introduced diffusion term makes the level sets of the TE more regular. This improves adversarial robustness of the classifier. Hence, we arrive at the following convection-diffusion equation
\begin{eqnarray}
\label{PDE-Eq8}
\begin{cases}
\frac{\partial u}{\partial t}(\mathbf{x}, t) + \overline{F}(\mathbf{x}, \mathbf{w}(t))\cdot \nabla u(\mathbf{x}, t) + \frac{1}{2}\sigma^2\Delta u(\mathbf{x}, t) = 0, \ \
\mathbf{x}\in \mathbb{R}^d,\ \ \  t\in [0, 1),\\
u(\mathbf{x}, 1) = f(\mathbf{x}).
\end{cases}
\end{eqnarray}
The solution of Eq.~(\ref{PDE-Eq8}) is much more regular when $\sigma\neq 0$ than when $\sigma=0$. We consider the solution of Eq.~(\ref{PDE-Eq8}) in a 2D unit square with periodic boundary conditions, and on each grid point of the mesh the velocity field $\overline{F}(\mathbf{x}, \mathbf{w}(t))$ is a random number sampled uniformly from $-1$ to $1$. The terminal condition is also randomly generated, as shown in Fig.~\ref{fig-Solution-Advection-Diffusion} (d). This 2D convection-diffusion equation is solved by the pseudo-spectral method with spatial and temporal step sizes being $1/128$ and $1\times 10^{-3}$, respectively. Figure~\ref{fig-Solution-Advection-Diffusion} (a), (b), and (c) illustrate the solutions when $\sigma=0$, $0.01$, and $0.1$, respectively. These show that as $\sigma$ increases, the solution becomes more regular, which makes the classifier more robust, but might be less accurate on clean data. The $\sigma$ should be selected to have a good trade-off between accuracy and robustness.
According to the above observation, instead of using $u(\mathbf{x}, 0)$ of the TE's solution for classification, we use that of the convection-diffusion equation.

\subsection{Theoretical Guarantees for the Surrogate Model}
We have the following theoretical guarantee for robustness of the solution of the convection-diffusion equation mentioned above.
\begin{theorem}\cite{Ladyzhenskaja:1968}\label{Theorem-Robust}
Let $\overline{F}(\mathbf{x}, t)$ be a Lipschitz function in both $\mathbf{x}$ and $t$, and $f(\mathbf{x})$ be a bounded function. Consider the following initial value problem of the convection-diffusion equation ($\sigma\neq 0$)
\begin{equation}\label{TE-Theorem}
\begin{cases}
\frac{\partial u}{\partial t}(\mathbf{x}, t) + \overline{F}(\mathbf{x}, \mathbf{w}(t))\cdot \nabla u(\mathbf{x}, t) + \frac{1}{2}\sigma^2\Delta u(\mathbf{x}, t) = 0, \ \
\mathbf{x}\in \mathbb{R}^d,\ \ \  t\in [0, 1),\\
u(\mathbf{x}, 1) = f(\mathbf{x}).
\end{cases}    
\end{equation}
Then, for any small perturbation $\delta$, we have $|u(\mathbf{x}+\delta, 0) - u(\mathbf{x}, 0)| \leq C\left(\frac{\|\delta\|_2}{\sigma}\right)^\alpha$ for some constant $\alpha>0$ if $\sigma \leq 1$. Here, $\|\delta\|_2$ is the $\ell_2$ norm of $\delta$, and $C$ is a constant that depends on $d$, $\|f\|_\infty$, and $\|\overline{F}\|_{L_{\mathbf{x}, t}^\infty}$. The meaning of notations $\|f\|_\infty$ and $\|\overline{F}\|_{L_{\mathbf{x}, t}^\infty}$ can be found in \cite{Ladyzhenskaja:1968}.
\end{theorem}

Furthermore, we have the following bound for the gradient of the solution of the convection-diffusion equation.
\begin{theorem}\label{Initial-Value-Problem}
Let $u_0(\mathbf{x})$ be a compactly supported function and $\overline{F}\in C^1(\mathbf{R}^d\times [0, 1])$. For the following initial value problem of the convection-diffusion equation
\begin{equation}\label{IVP}
\begin{cases}
\frac{\partial u}{\partial t}(\mathbf{x}, t) + \overline{F}(\mathbf{x}, \mathbf{w}(t))\cdot \nabla u(\mathbf{x}, t) = \frac{1}{2}\sigma^2\Delta u(\mathbf{x}, t), \ \
\mathbf{x}\in \mathbb{R}^d,\ \ \  t\in [0, 1),\\
u(\mathbf{x}, 0) = u_0(\mathbf{x}),
\end{cases}    
\end{equation}
we have
\begin{equation}
\label{IVP-Estimate}
\|\nabla u(\mathbf{x}, 1)\|_\infty \leq e^{-\sigma^2}e^\gamma \left(\|u_0\|_\infty + \|\nabla u_0\|_\infty\right),
\end{equation}
where $\gamma$ is a constant depends on $\nabla \overline{F}$.
\end{theorem}

\begin{proof}
Let $w(\mathbf{x}, t)=\left(\mu u^2(\mathbf{x}, t) + \|\nabla u(\mathbf{x}, t)\|^2\right)e^{-2\lambda t}$, where $\mu$ and $\lambda$ are constants which will be defined later.

\noindent Note that $u^2(\mathbf{x}, t)$ satisfies
$$
\frac{\partial (u^2)}{\partial t} + \overline{F}(\mathbf{x}, t)\cdot \nabla(u^2) = \sigma^2\Delta(u^2) - 2\sigma^2\|\nabla u\|^2,
$$
and $\|\nabla u\|^2$ satisfies
$$
\frac{\partial \|\nabla u\|^2}{\partial t} + \overline{F}(\mathbf{x}, t)\cdot \nabla\|\nabla u\|^2=-2\nabla u\cdot \nabla \overline{F}\cdot \nabla u +\sigma^2\Delta \|\nabla u\|^2 - 2\sigma^2\|\nabla\nabla u\|_F^2,
$$
therefore,
{\small
$$
\frac{\partial w}{\partial t} + \overline{F}\cdot \nabla w - \sigma^2\nabla w = e^{-2\lambda t}\left[ -2\lambda (\mu u^2+\|\nabla u\|^2) - 2\mu\sigma^2\|\nabla u\|^2 -2\nabla u\cdot \nabla \overline{F}\cdot \nabla u -2\sigma^2 \|\nabla\nabla u\|_F^2\right].
$$}
Next, let $\gamma(\mathbf{x}, t) = \min_{\|\xi\|=1}\xi\cdot\nabla \overline{F}\cdot \xi$ and $\gamma = -\min_{\mathbf{x}, t}\gamma(\mathbf{x}, t)$, then we have
$$
Lw:=\frac{\partial w}{\partial t} + \overline{F}\cdot \nabla w - \sigma^2\nabla w \leq -2e^{-2\lambda t}\left[ \lambda\mu u^2 + (\lambda + \mu\sigma^2-\gamma)\|\nabla u\|^2 \right].
$$
If we choose $\lambda$ and $\mu$ large enough, such that $\lambda+\mu\sigma-\gamma\geq 0$, then 
$$
Lw \leq 0.
$$
From the maximum principle, we know $\max_\mathbf{x}w(\mathbf{x}, 1)\leq \max_\mathbf{x}w(\mathbf{x}, 0)$, i.e.,
$$
\max_\mathbf{x} e^{-2\lambda}\left(\mu u^2(\mathbf{x}, 1) + \|\nabla u(\mathbf{x}, 1)\|^2 \right)\leq \max_\mathbf{x}\left( \mu u^2(\mathbf{x}, 0) + \|\nabla u(\mathbf{x}, 0)\|^2\right).
$$
Hence,
$$
\|\nabla u(\mathbf{x}, 1)\|_\infty^2 \leq e^{2\lambda}\left(\mu\|u_0\|_\infty^2 + \|\nabla u_0\|_\infty^2 \right).
$$
Let $\mu=1$ and $\lambda=\gamma - \sigma^2$, we have
$$
\|\nabla u(\mathbf{x}, 1)\|_\infty \leq e^{-\sigma^2}\left(e^\gamma\|u_0\|_\infty + e^\gamma\|\nabla u_0\|_\infty\right).
$$
\end{proof}

\begin{remark}
Similar estimate in Theorem~\ref{Initial-Value-Problem} can be established on $u(\mathbf{x}, 0)$ for the terminal value problem of the convection diffusion equation in Eq.~(\ref{TE-Theorem}) by reverse time.
\end{remark}

\section{Algorithms}\label{Section-Algorithms}
\subsection{ResNets Ensemble via the Feynman-Kac Formula}
Based on the above discussion, if we use the solution of the convection-diffusion equation, Eq.~(\ref{PDE-Eq8}), for classification. The resulted classifier will be more resistant to adversarial attacks. In this part, we will present an ensemble of ResNets to approximate the solution of Eq.~(\ref{PDE-Eq8}). In the Section.~\ref{Section-Results}, we will verify that the robustly trained special ensemble of ResNets is more accurate on both clean and adversarial images than standard ResNets.
\medskip

The convection-diffusion equation, Eq.~(\ref{PDE-Eq8}), can be solved using the Feynman-Kac formula \cite{FCFormula:1949} in high dimensional space, which gives  $u(\hat{\mathbf{x}}, 0)$ as
\begin{equation}
\label{PDE-Eq9}
u(\hat{\mathbf{x}}, 0) = \mathbb{E}\left[f(\mathbf{x}(1))|\mathbf{x}(0)=\hat{\mathbf{x}}\right],
\end{equation}
where $\mathbf{x}(t)$ is an It\^{o} process,
$$d\mathbf{x}(t) = \overline{F}(\mathbf{x}(t), \mathbf{w}(t))dt + \sigma dB_t,$$
and $u(\hat{\mathbf{x}}, 0)$ is the conditional expectation of $f(\mathbf{x}(1))$.
\medskip

Next, we approximate the Feynman-Kac formula by an ensemble of modified ResNets in the following way:
Accoding to the Euler-Maruyama method \cite{Kloeden:1992}, the term $\sigma dB_t$ in the It\^{o} process that can be approximated by adding a specially designed Gaussian noise, 
$\sigma \mathcal{N}(\mathbf{0}, \mathbf{I})$, where $\sigma = a\sqrt{{\rm Var}(\mathbf{x}_l+\mathcal{F}(\mathbf{x}_l))}$ with $a$ being a tunable parameter, 
to each original residual mapping $\mathbf{x}_{l+1} = \mathbf{x}_l + \mathcal{F}(\mathbf{x}_l)$ in the ResNet. This gives the modified residual mapping $\mathbf{x}_{l+1} = \mathbf{x}_l + \mathcal{F}(\mathbf{x}_l) + \sigma \mathcal{N}(\mathbf{0}, \mathbf{I})$, as illustrated in Fig.~\ref{fig:Structure} (b).
Let ResNet' denote the modified ResNet where we inject noise to each residual mapping of the original ResNet. In a nutshell, ResNet's approximation to the Feynman-Kac formula is an ensemble of jointly trained ResNet' as illustrated in Fig.~\ref{fig:Structure} (c). \footnote{To ease the notation, in what follows, we use ResNet in place of ResNet' when there is no ambiguity.}
We call this ensemble of ResNets as EnResNet. For instance, if the base ResNet is ResNet20, an ensemble of $n$ ResNet20 is denoted as En$_n$ResNet20.

\subsection{Adversarial Attacks}
In this subsection, we review a few widely used adversarial attacks. These attacks will be used to train robust EnResNets and attack the trained models. We attack the trained model, $f(\mathbf{x}, \mathbf{w})$, by $\ell_\infty$ norm based (the other norm based attacks can be formulated similarly) untargeted fast gradient sign method (FGSM), iterative FGSM (IFGSM) \cite{Goodfellow:2014AdversarialTraining}, and Carlini-Wagner (C\&W) \cite{CWAttack:2016} attacks in both white-box and blind fashions. In blind attacks, we use the target model to classify the adversarial images crafted by attacking the oracle model in a white-box approach. For a given instance ($\mathbf{x}$, $y$):
\begin{itemize}
\item FGSM searches the adversarial image $\mathbf{x}'$ by maximizing the loss function
$\mathcal{L}(\mathbf{x}', y) \doteq \mathcal{L}(f(\mathbf{x}', \mathbf{w}), y)$,
subject to the constraint $||\mathbf{x}'-\mathbf{x}||_\infty \leq \epsilon$ with $\epsilon$ being the maximum perturbation. For the linearized loss function, $\mathcal{L}(\mathbf{x}', y) \approx \mathcal{L}(\mathbf{x}, y) + \nabla_x\mathcal{L}(\mathbf{x}, y)^T \cdot (\mathbf{x}' - \mathbf{x})$,
the optimal adversarial is
\begin{equation}\label{FGSM}
\mathbf{x}'=\mathbf{x} + \epsilon \cdot {\rm sign} \left( \nabla_\mathbf{x}\mathcal{L}(\mathbf{x}, y) \right).
\end{equation}
\item IFGSM, Eq.~(\ref{IFGSM-2}), iterates FGSM with step size $\alpha$ and clips the perturbed image to generate the enhanced adversarial attack, 
\begin{equation}
\label{IFGSM-2}
\mathbf{x}^{(m)} = {\rm Clip}_{\mathbf{x}, \epsilon}\{\mathbf{x}^{(m-1)} + \alpha \cdot {\rm sign} ( \nabla_{\mathbf{x}} \mathcal{L}(\mathbf{x}^{(m-1)}, y) )\},
\end{equation}
where $m=1, \cdots, M$, $\mathbf{x}^{(0)} = \mathbf{x}$, and let the adversarial image be $\mathbf{x}'=\mathbf{x}^{(M)}$ with $M$ being the total number of iterations.
\item C\&W attack searches the targeted adversarial image by solving 
\begin{equation}
\label{cwl2-eq1}
\min_{\delta} ||\delta||_\infty,\ \ \mbox{subject to}\ \ f(\mathbf{w}, \mathbf{x}+\delta) = t, \; \mathbf{x}+\delta \in [0, 1]^d,
\end{equation}
where $\delta$ is the adversarial perturbation and $t$ is the target label. 
Carlini et al. \cite{CWAttack:2016} proposed the following approximation to Eq.~(\ref{cwl2-eq1}),
\begin{eqnarray}
\label{CWL2}
\min_{\mathbf{u}} ||\frac{1}{2}\left(\tanh(\mathbf{u}) + 1\right) - \mathbf{x} ||_\infty + \\ \nonumber
c\cdot \max\left\{-\kappa, \max_{i\neq t}(Z(\frac{1}{2}(\tanh(\mathbf{u}))+1)_i)
- Z(\frac{1}{2}(\tanh(\mathbf{u}))+1)_t \right\},
\end{eqnarray}
where $Z(\cdot)$ is the logit vector for the input, i.e., the output of the DNN before the softmax layer. This unconstrained optimization problem 
can be solved efficiently by using the Adam optimizer \cite{Kingma:2014Adam}. Dou et al. \cite{Dou:2018}, prove that, under a certain regime, C\&W can shift the DNNs' predicted probability distribution to the desired one.
\end{itemize}
\medskip

All three attacks clip the pixel values of the adversarial image to between 0 and 1. In the following experiments, we set $\epsilon = 8/255$ in both FGSM and IFGSM attacks. Additionally, in IFGSM we set $m=20$ and $\alpha=2/255$, and denote it as IFGSM$^{20}$. For C\&W attack, we run 50 iterations of Adam with learning rate $6\times 10^{-4}$ and set $c=10$ and $\kappa=0$.

\subsection{Robust Training of EnResNets}
We use the PGD adversarial training \cite{Madry:2018}, i.e., solving EARM Eq.~(\ref{Adversarial-ERM}) by replacing $\mathbf{x}'$ with the PGD adversarial one, to robustly train EnResNets with $\sigma=0.1$ on both CIFAR10 and CIFAR100 \cite{Alex:2009-Cifar} benchmarks with standard data augmentation \cite{ResNet}. The attack in the PGD adversarial training is merely IFGSM with an initial random perturbation on the clean data. We summarize the PGD based robust training for EnResNets in Algorithm~\ref{PGD-EnResNetTraining}. Other methods to solve EARM can also be used to train EnResNets, e.g., approximation to the adversarial risk function and regularization. EnResNet enriches the hypothesis class $\mathcal{H}$, to make the classifiers from $\mathcal{H}$ more adversarially robust. All computations are carried out on a machine
with a single Nvidia Titan Xp graphics card.
\medskip

\begin{algorithm}[h] 
\caption{Training of the EnResNet by PGD Adversarial Training}\label{PGD-EnResNetTraining}
\begin{algorithmic}
\State \textbf{Input:} Training set: $(\mathbf{X}_i, \mathbf{Y}_i)_{i=1}^{N_B}$, $N_B = \# {\rm minibatches}$, perturbation $\epsilon$, and step size $\alpha$.
\State \textbf{Output: } A robustly trained En$_N$ResNet, i.e., an ensemble of $N$ modified ResNets.
\For {$i = 1, \dots, N_E$ (where $N_E$ is the number of epochs.)}
\For{$j = 1, $\dots$, N_B$}
\State //{\it PGD attack}
\State Add uniform noise in the range $[-\epsilon, \epsilon]$ to $\mathbf{X}_i$, denote the resulted images as $\tilde{\mathbf{X}}_i$.
\State Attack $\tilde{\mathbf{X}}_i$ by 10 iterations IFGSM attacks with maximum perturbation $\epsilon$ and step size $\alpha$. And denote the adversarial images as $\mathbf{X}'_i$.
\State //{\it Forward-propagation}
\State Generate prediction $\tilde{\mathbf{Y}}_i = {\rm En}_N{\rm ResNet}(\mathbf{X}'_i)$ for $\mathbf{X}'_i$ by the current model En$_N$ResNet.
\State //{\it Back-propagation}
\State Back-Propagate the cross-entropy loss between $\mathbf{Y}_i$ and $\tilde{\mathbf{Y}}_i$ to update the model EnResNet$_N$.
\EndFor
\EndFor
\end{algorithmic}
\end{algorithm}

\section{Numerical Results}\label{Section-Results}
In this section, we numerically verify that the robustly trained EnResNets are more accurate, on both clean and adversarial data of the CIFAR10 and CIFAR100, than robustly trained ResNets and ensemble of ResNets without noise injection. To avoid the gradient mask issue of EnResNets due to the noise injection in each residual mapping, we use the Expectation over Transformation (EOT) strategy \cite{Athalye:2018B} to compute the gradient which is averaged over five independent runs.

\subsection{Natural and Robust Accuracies of Robustly Trained EnResNets}
In robust training, we run 200 epochs of the PGD adversarial training (10 iterations of IFGSM with $\alpha=2/255$ and $\epsilon=8/255$, and an initial random perturbation of magnitude $\epsilon$) with initial learning rate $0.1$, 
which decays by a factor of $10$ at the 80th, 120th, and 160th epochs. The training data is split into 45K/5K for training and validation, the model with the best validation accuracy is used for testing. Similar settings are used for natural training, i.e., solving the ERM problem Eq.~(\ref{ERM}). En$_1$ResNet20 denotes the ensemble of only one ResNet20 which is merely adding noise to each residual mapping, and similar notations apply to other DNNs.
\medskip

First, we show that the ensemble of noise injected ResNets can improve the natural generalization of the naturally trained models. As shown in Table~\ref{ERM-Cifar10}, the naturally trained ensemble of multiple ResNets are always generalize better on the clean images than the base ResNets. This conclusion is verified by ResNet20, ResNet44, and ResNet110. However, the natural accuracy of the robustly trained models are much less than that of the naturally trained models. For instance, the natural accuracies of the robustly trained and naturally trained ResNet20 are, respectively, $75.11$\% and $92.10$\%. The degradation of natural accuracies in robust training are also confirmed by experiments on ResNet44 ($78.89$\% v.s. $93.22$\%) and ResNet110 ($82.19$\% v.s. $94.30$\%). Improving natural accuracy of the robustly trained models is another important issue during adversarial defense.
\medskip

\begin{table}[tp]
\centering
\fontsize{8.5}{8.5}\selectfont
\begin{threeparttable}
\caption{Natural accuracies of naturally trained ResNet20 and different ensemble of noise injected ResNet20 on the CIFAR10 dataset. Unit: \%.}\label{ERM-Cifar10}
\begin{tabular}{ccc}
\toprule[1.0pt]
\ \ \ \ \ \ \ \ \ \ \ \ \ \ \ \ \ \ \  Model\ \ \ \ \ \ \ \ \ \ \ \ \ \ \ \ \ \ \  &\ \ \ \ \ \ \ \ \ \ \ \ \ \ \ \ \ \ \   dataset\ \ \ \ \ \ \ \ \ \ \ \ \ \ \   &\ \ \ \ \ \ \ \ \ \ \ \ \ \ \   $\mathcal{A}_{\rm nat}$\ \ \ \ \ \ \ \ \ \ \ \ \ \ \ \    \cr
\midrule[0.8pt]
ResNet20        & CIFAR10  & 92.10\cr
En$_1$ResNet20  & CIFAR10  & 92.59\cr
En$_2$ResNet20  & CIFAR10  & 92.60\cr
En$_5$ResNet20  & CIFAR10  & 92.74\cr
ResNet44        & CIFAR10  & 93.22\cr
En$_1$ResNet44  & CIFAR10  & 93.37\cr
En$_2$ResNet44  & CIFAR10  & 93.54\cr
ResNet110       & CIFAR10  & 94.30\cr
En$_2$ResNet110 & CIFAR10  & 93.49\cr
\bottomrule[1.0pt]
\end{tabular}
\end{threeparttable}
\end{table}

Second, consider natural ($\mathcal{A}_{\rm nat}$) and robust ($\mathcal{A}_{\rm rob}$) accuracies of the PGD adversarially trained models on the CIFAR10, where $\mathcal{A}_{\rm nat}$ and $\mathcal{A}_{\rm rob}$ are measured on clean and adversarial images, respectively. All results are listed in Table~\ref{EARM-Cifar10}. 
The robustly trained ResNet20 has accuracies $50.89$\%, $46.03$\% (close to that reported in \cite{Madry:2018}), and $58.73$\%, respectively, under the FGSM, IFGSM$^{20}$, and C\&W attacks. Moreover, it has a natural accuracy of $75.11$\%. En$_5$ResNet20 boosts natural accuracy to $82.52$\%, and improves the corresponding robust accuracies to $58.92$\%, $51.48$\%, and $67.73$\%, respectively. Simply injecting noise to each residual mapping of ResNet20 
can increase $\mathcal{A}_{\rm nat}$ by $\sim 2\%$ and $\mathcal{A}_{\rm rob}$ by $\sim 3\%$ under the IFGSM$^{20}$ attack. 
The advantages of EnResNets are also verified by experiments on ResNet44, ResNet110, and their ensembles. 
Note that ensemble of high capacity ResNet is more robust than low capacity model: as shown in Table~\ref{EARM-Cifar10}, En$_2$ResNet110 is more accurate than En$_2$ResNet44 which in turn is more accurate than En$_2$ResNet20 in classifying both clean and adversarial images. 
The robustly trained En$_1$WideResNet34-10 has $86.19$\% and $56.60$\%, respectively, natural and robust accuracies under the IFGSM$^{20}$ attack. Compared with the current state-of-the-art \cite{Zhang:2019-Trades}, En$_1$WideResNet34-10 has almost the same robust accuracy ($56.60\%$ v.s. $56.61\%$) under the IFGSM$^{20}$ attack but better natural accuracy ($86.19\%$ v.s. $84.92\%$). Figure~\ref{fig-ResNet-Acc-Evolution} plots the evolution of training and validation accuracies of ResNet20 and ResNet44 and their different ensembles.
\medskip

\begin{table}[!h]
\centering
\fontsize{8.5}{8.5}\selectfont
\begin{threeparttable}
\caption{Natural and robust accuracies of different base and noise injected ensembles of robustly trained ResNets on the CIFAR10. Unit: \%.}\label{EARM-Cifar10}
\begin{tabular}{cccccc}
\toprule[1.0pt]
Model & dataset &\ \ \ $\mathcal{A}_{\rm nat}$\ \ \  &\ \ \  $\mathcal{A}_{\rm rob}$ (FGSM)\ \ \  &\ \ \  $\mathcal{A}_{\rm rob}$ (IFGSM$^{20}$)\ \ \  &\ \ \  $\mathcal{A}_{\rm rob}$ (C\&W)\ \ \  \cr
\midrule[0.8pt]
ResNet20        & CIFAR10  &  75.11  & 50.89  & 46.03 & 58.73  \cr
En$_1$ResNet20  & CIFAR10  &  77.21  & 55.35  & 49.06 & 65.69  \cr
En$_2$ResNet20  & CIFAR10  &  80.34  & 57.23  & 50.06 & 66.47  \cr
En$_5$ResNet20  & CIFAR10  &  82.52  & 58.92  & 51.48 & 67.73  \cr
ResNet44        & CIFAR10  &  78.89  & 54.54  & 48.85 & 61.33  \cr
En$_1$ResNet44  & CIFAR10  &  82.03  & 57.80  & 51.83 & 66.00  \cr
En$_2$ResNet44  & CIFAR10  &  82.91  & 58.29  & 51.86 & 66.89  \cr
ResNet110       & CIFAR10  &  82.19  & 57.61  & 52.02 & 62.92  \cr
En$_2$ResNet110 & CIFAR10  &  82.43  & 59.24  & 53.03 & 68.67  \cr
En$_1$WideResNet34-10 & CIFAR10 & {\bf 86.19} & {\bf 61.82}  & {\bf 56.60} & {\bf 69.32} \cr
\bottomrule[1.0pt]
\end{tabular}
\end{threeparttable}
\end{table}

\begin{figure}
\centering
\begin{tabular}{cc}
\includegraphics[width=0.46\columnwidth]{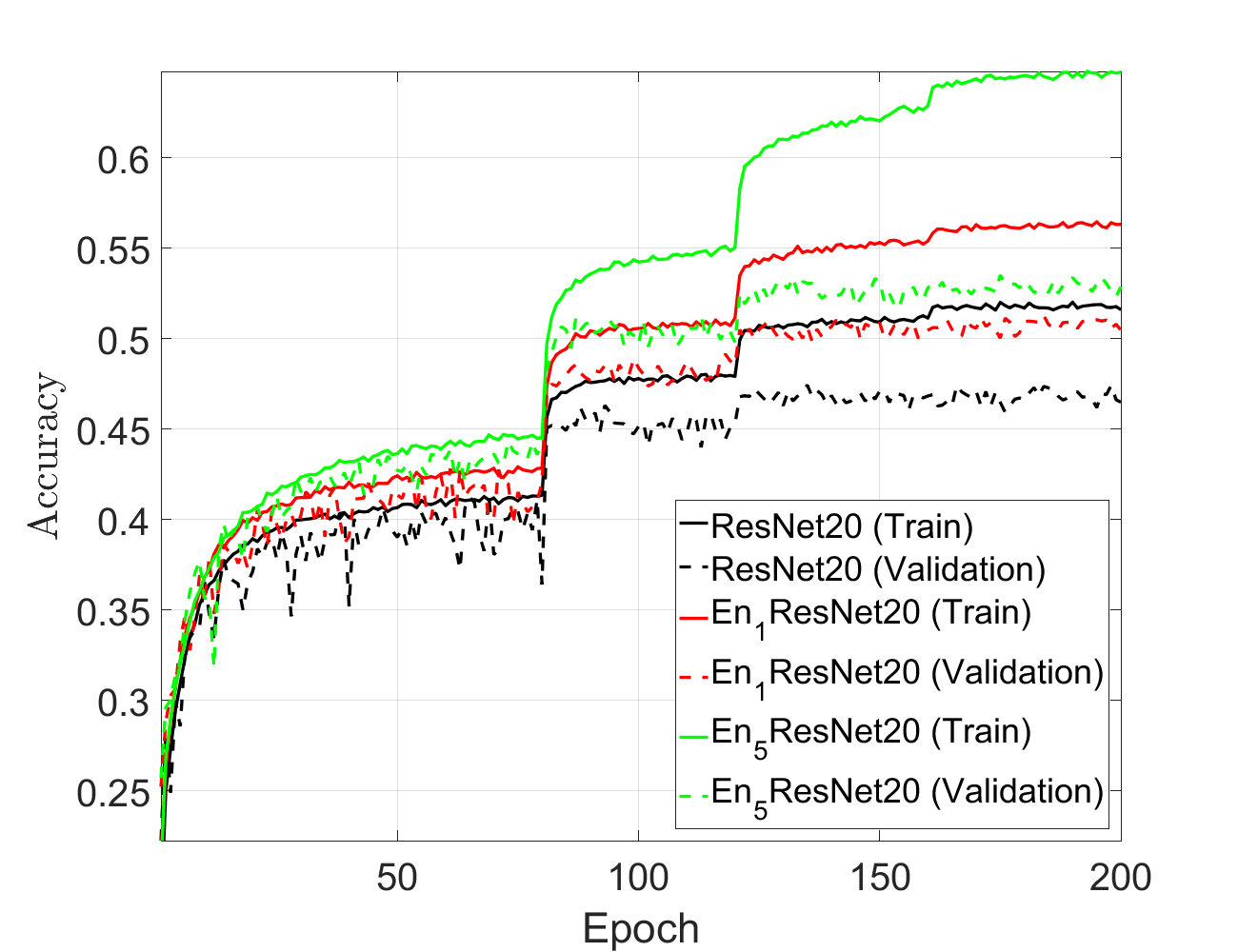}&
\includegraphics[width=0.46\columnwidth]{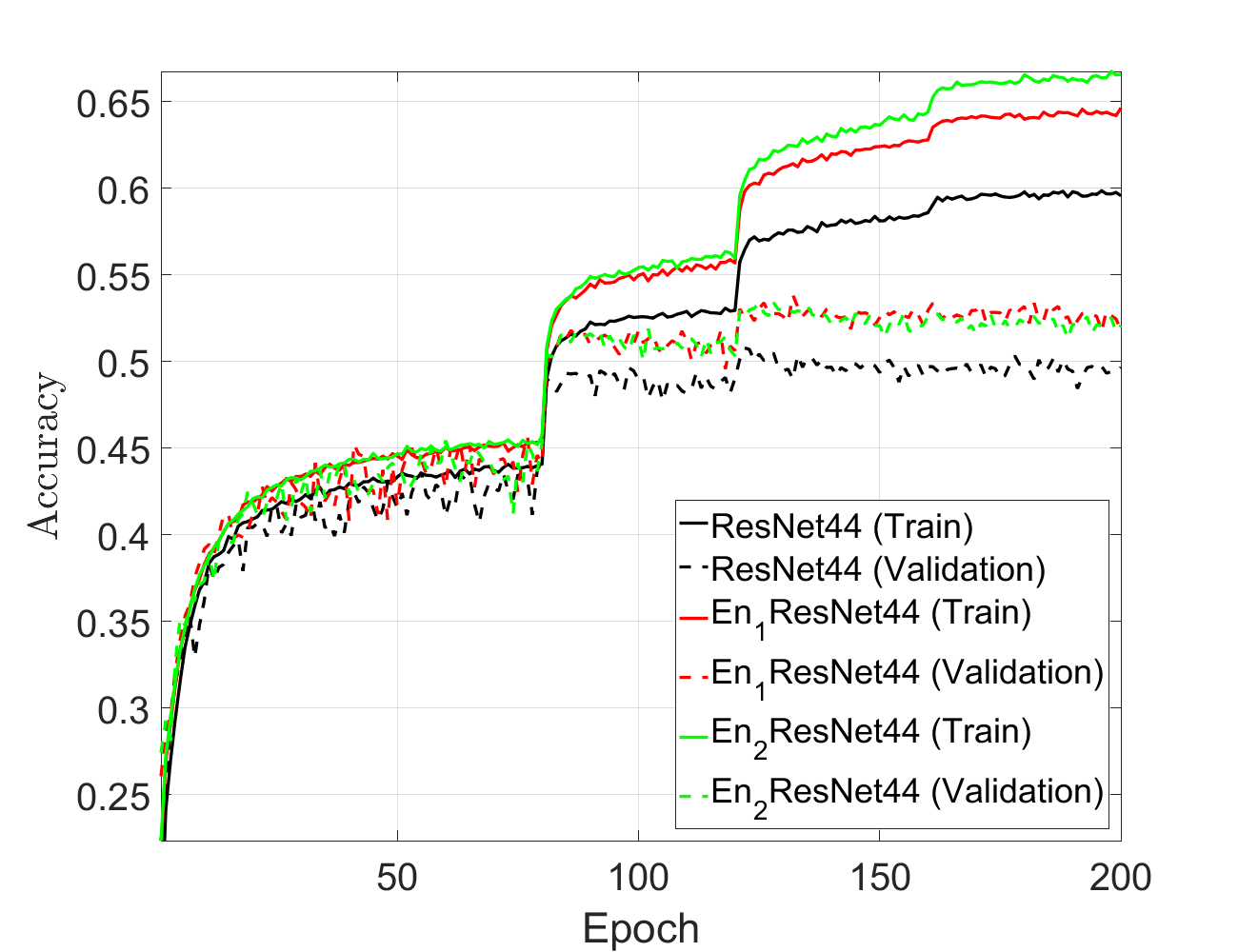}\\
(a) & (b)\\
\end{tabular}
\caption{Evolution of training and validation accuracy. (a): ResNet20 and different ensembles of noise injected ResNet20. (b): ResNet44 and different ensembles of noise injected ResNet44.}
\label{fig-ResNet-Acc-Evolution}
\end{figure}

Third, consider accuracy of the robustly trained models under blind attacks. In this scenario, we use the target model to classify the adversarial images crafted by applying FGSM, IFGSM$^{20}$, and C\&W attacks to the oracle model. As listed in Table~\ref{EARM-Cifar10-BlindAttack}, EnResNets are always more robust than the base ResNets under different blind attacks. For instance, when En$_5$ResNet20 is used to classify adversarial images crafted by attacking ResNet20 with FGSM, IFGSM$^{20}$, and C\&W attacks, the accuracies are $64.07$\%, $62.99$\%, and $76.57$\%, respectively. Conversely, the accuracies of ResNet20 are only $61.69$\%, $58.74$\%, and $73.77$\%, respectively, in classifying adversarial images obtained by using the above three attacks to attack En$_5$ResNet20.
\medskip

\begin{table}[!h]
\centering
\fontsize{8.5}{8.5}\selectfont
\begin{threeparttable}
\caption{Accuracies of robustly trained models on adversarial images of CIFAR10 crafted by attacking the oracle model with different attacks. Unit: \%.}\label{EARM-Cifar10-BlindAttack}
\begin{tabular}{cccccc}
\toprule[1.0pt]
\ \ \ Model\ \ \  &\ \ \  dataset\ \ \  &\ \ \  Oracle \ \ \  &\ \ \  $\mathcal{A}_{\rm rob}$ (FGSM)\ \ \  & \ \ \ $\mathcal{A}_{\rm rob}$ (IFGSM$^{20}$)\ \ \  &\ \ \  $\mathcal{A}_{\rm rob}$ (C\&W)\ \ \ \cr
\midrule[0.8pt]
ResNet20        & CIFAR10  & En$_5$ResNet20    & 61.69  & 58.74 & 73.77  \cr
En$_5$ResNet20  & CIFAR10  & ResNet20          & 64.07  & 62.99 & 76.57  \cr
ResNet44        & CIFAR10  & En$_2$ResNet44    & 63.87  & 60.66 & 75.83  \cr
En$_2$ResNet44  & CIFAR10  & ResNet44          & 64.52  & 61.23 & 76.99  \cr
ResNet110       & CIFAR10  & En$_2$ResNet110   & 64.19  & 61.80 & 75.19  \cr
En$_2$ResNet110 & CIFAR10  & ResNet110         & 66.26  & 62.89 & 77.71  \cr
\bottomrule[1.0pt]
\end{tabular}
\end{threeparttable}
\end{table}

Fourth, we perform experiments on the CIFAR100 to further verify the efficiency of EnResNets in defending against adversarial attacks. Table~\ref{ERM-Cifar100} lists the naturally accuracies of the naturally trained ResNets and their ensembles, again, the ensemble can improve natural accuracies. Table~\ref{EARM-Cifar100} lists natural and robust accuracies of robustly trained ResNet20, ResNet44, and their ensembles under white-box attacks. The robust accuracy under the blind attacks is listed in Table~\ref{EARM-Cifar100-BlindAttack}. The natural accuracy of the PGD adversarially trained baseline ResNet20 is $46.02$\%, and it has robust accuracies $24.77$\%, $23.23$\%, and $32.42$\% under FGSM, IFGSM$^{20}$, and C\&W attacks, respectively. En$_5$ResNet20 increases them to $51.72$\%, $31.64$\%, $27.80$\%, and $40.44$\%, respectively. The ensemble of ResNets is more effective in defending against adversarial attacks than making the ResNets deeper. For instance, En$_2$ResNet20 that has $\sim 0.27M \times 2$ parameters is much more robust to adversarial attacks, FGSM ($30.20$\% v.s. $28.40$\%), IFGSM$^{20}$ ($26.25$\% v.s. $25.81$\%), and C\&W ($40.06$\% v.s. $36.06$\%), than ResNet44 with $\sim 0.66M$ parameters. Under blind attacks, En$_2$ResNet20 is also significantly more robust to different attacks where the opponent model is used to generate adversarial images. Under the same model and computation complexity, EnResNets is more robust to adversarial images and more accurate on clean images than deeper nets.
\medskip

\begin{table}[tp]
\centering
\fontsize{8.5}{8.5}\selectfont
\begin{threeparttable}
\caption{Natural accuracies of naturally trained ResNet20 and different ensemble of noise injected ResNet20 on the CIFAR100 dataset. Unit: \%.}\label{ERM-Cifar100}
\begin{tabular}{ccc}
\toprule[1.0pt]
\ \ \ \ \ \ \ \ \ \ \ \ \ \ \ \ \ \ \  Model\ \ \ \ \ \ \ \ \ \ \ \ \ \ \ \ \ \ \  &\ \ \ \ \ \ \ \ \ \ \ \ \ \ \ \ \ \ \   dataset\ \ \ \ \ \ \ \ \ \ \ \ \ \ \   &\ \ \ \ \ \ \ \ \ \ \ \ \ \ \   $\mathcal{A}_{\rm nat}$\ \ \ \ \ \ \ \ \ \ \ \ \ \ \ \    \cr
\midrule[0.8pt]
ResNet20        & CIFAR100  & 68.53\cr
ResNet44        & CIFAR100  & 71.48\cr
En$_2$ResNet20  & CIFAR100  & 69.57\cr
En$_5$ResNet20  & CIFAR100  & 70.22\cr
\bottomrule[1.0pt]
\end{tabular}
\end{threeparttable}
\end{table}

\begin{table}[!h]
\centering
\fontsize{8.5}{8.5}\selectfont
\begin{threeparttable}
\caption{Natural and robust accuracies of robustly trained ResNet20 and different ensemble of noise injected ResNet20 on the CIFAR100. Unit: \%.
}\label{EARM-Cifar100}
\begin{tabular}{cccccc}
\toprule[1.0pt]
\ \ \ \ Model\ \ \ \  &\ \ \ \  dataset\ \ \ \  &\ \ \ \  $\mathcal{A}_{\rm nat}$\ \ \ \  &\ \ \ \ $\mathcal{A}_{\rm rob}$ (FGSM)\ \ \ \  &\ \ \ \  $\mathcal{A}_{\rm rob}$ (IFGSM$^{20}$)\ \ \ \  &\ \ \ \ $\mathcal{A}_{\rm rob}$ (C\&W)\ \ \ \ \cr
\midrule[0.8pt]
ResNet20        & CIFAR100  & 46.02   & 24.77  & 23.23 & 32.42  \cr
En$_2$ResNet20  & CIFAR100  & 50.68   & 30.20  & 26.25 & 40.06  \cr
En$_5$ResNet20  & CIFAR100  & {\bf 51.72}   & {\bf 31.64}  & {\bf 27.80} & {\bf 40.44}  \cr
ResNet44        & CIFAR100  & 50.38   & 28.40  & 25.81 & 36.06  \cr
\bottomrule[1.0pt]
\end{tabular}
\end{threeparttable}
\end{table}

\begin{table}[!h]
\centering
\fontsize{8.5}{8.5}\selectfont
\begin{threeparttable}
\caption{Accuracies of robustly trained models on the adversarial images of CIFAR100 crafted by attacking the oracle model with different attacks. Unit: \%.}\label{EARM-Cifar100-BlindAttack}
\begin{tabular}{cccccc}
\toprule[1.0pt]
\ \ \ Model\ \ \  &\ \ \  dataset\ \ \  &\ \ \  Oracle \ \ \  &\ \ \  $\mathcal{A}_{\rm rob}$ (FGSM)\ \ \  & \ \ \ $\mathcal{A}_{\rm rob}$ (IFGSM$^{20}$)\ \ \  &\ \ \  $\mathcal{A}_{\rm rob}$ (C\&W)\ \ \ \cr
\midrule[0.8pt]
ResNet20        & CIFAR100   & En$_2$ResNet20    & 33.08  & 30.79 & 41.52  \cr
En$_2$ResNet20  & CIFAR100   & ResNet20          & 34.15  & 33.34 & 48.21  \cr
\bottomrule[1.0pt]
\end{tabular}
\end{threeparttable}
\end{table}

Figure~\ref{fig-Attack} depicts a few selected images from the CIFAR10 and their adversarial ones crafted by applying either IFGSM$^{20}$ or C\&W attack to attack both ResNet20 and En$_5$ResNet20. Both adversarially trained ResNet20 and En$_5$ResNet20 fail to correctly classify any of the adversarial versions of these four images. For the deer image, it might also be difficult for human to distinguish it from a horse.
\begin{figure}[!h]
\centering
\begin{tabular}{c}
\includegraphics[width=0.50\columnwidth]{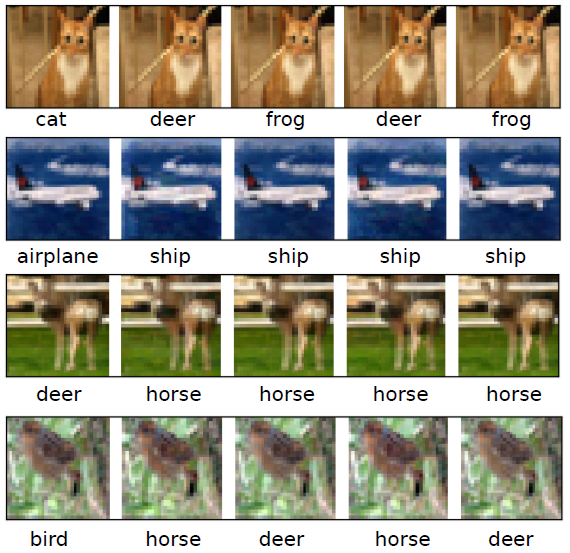}
\end{tabular}
\caption{Column 1: original images and labels; column 2-3 (4-5): adversarial images crafted by using IFGSM$^{20}$ and C\&W to attack ResNet20 (En$_5$ResNet20) and corresponding predicted labels.}
\label{fig-Attack}
\end{figure}

\subsection{Integration of Separately Trained EnResNets}
In the previous subsection, we verified the adversarial defense capability of EnResNet, which is an approximation to the Feynman-Kac formula to solve the convection-diffusion equation. As we showed, when more ResNets and larger models are involved in the ensemble, both natural and robust accuracies are improved. However, EnResNet proposed above requires to train the ensemble jointly, which poses memory challenges for training ultra-large ensembles. To overcome this issue, we consider training each component of the ensemble individually and integrating them together for prediction. The major benefit of this strategy is that with the same amount of GPU memory, we can train a much larger model for inference since the batch size used in inference can be one.
\medskip

Table~\ref{EARM-Cifar10-Separate-Trained} lists natural and robust accuracies of the integration of separately trained EnResNets on the CIFAR10. The integration of separately trained EnResNets 
have better robust accuracy than each component. For instance, the integration of En$_2$ResNet110 and En$_1$WideResNet34-10 gives a robust accuracy ${\bf 57.94}$\% under the IFGSM$^{20}$ attack, which is remarkably better than both En$_2$ResNet110 ($53.05$\%) and En$_1$WideResNet34-10 ($56.60$\%). To the best of our knowledge, $57.94$\% outperforms the current state-of-the-art \cite{Zhang:2019-Trades} by $1.33$\%. The effectiveness of the integration of separately trained EnResNets sheds light on the development of ultra-large models to improve efficiency for adversarial defense.

\begin{table}[!h]
\centering
\fontsize{8.5}{8.5}\selectfont
\begin{threeparttable}
\caption{Natural and robust accuracies of different integration of different robustly trained EnResNets on the CIFAR10. Unit: \%.
}\label{EARM-Cifar10-Separate-Trained}
\begin{tabular}{cccccc}
\toprule[1.0pt]
Model &dataset  & $\mathcal{A}_{\rm nat}$  & $\mathcal{A}_{\rm rob}$ (FGSM)  & $\mathcal{A}_{\rm rob}$ (IFGSM$^{20}$) & $\mathcal{A}_{\rm rob}$ (C\&W)\cr
\midrule[0.8pt]
En$_2$ResNet20\&En$_5$ResNet20  & CIFAR10  &  82.82  & 59.14  & 53.15 & 68.00  \cr
En$_2$ResNet44\&En$_5$ResNet20  & CIFAR10  &  82.99  & 59.64  & 53.86 & 69.36  \cr
En$_2$ResNet110\&En$_5$ResNet20 & CIFAR10  &  83.57  & 60.63  & 54.87 & 70.02  \cr
En$_2$ResNet110\&En$_1$WideResNet34-10 & CIFAR10  &  {\bf 85.62}  &  {\bf 62.48} & {\bf 57.94} & {\bf 70.20} \cr
\bottomrule[1.0pt]
\end{tabular}
\end{threeparttable}
\end{table}

\subsection{Comparison with the Wide ResNet}
In this subsection, we show that with the same number of parameters, EnResNets is more adversarially robust that the Wide ResNets. We compare EnResNet$_2$20 with the wide-ResNet: WRN-14-2 \cite{Zagoruyko2016WRN}. WRN-14-2 has $\sim0.69$M parameters which is more than that of EnResNet$_2$20. We list natural and robust accuracies of the robustly trained models on the CIFAR10 benchmark in Table.~\ref{EARM-Cifar10-WideResNet}. 
En$_2$ResNet20 has higher natural accuracy than WRN-14-2 ($80.34$\% v.s. $78.37$\%). Moreover, En$_2$ResNet20 is more robust to both IFGSM$^{20}$ and C\&W attacks.

\begin{table}[tp]
\centering
\fontsize{8.5}{8.5}\selectfont
\begin{threeparttable}
\caption{Natural and robust accuracies of robustly trained En$2$esNet20 and WRN-14-2 on the CIFAR10 dataset. Unit: \%.}\label{EARM-Cifar10-WideResNet}
\begin{tabular}{cccccc}
\toprule[1.0pt]
Model & dataset & $\mathcal{A}_{\rm nat}$  & $\mathcal{A}_{\rm rob}$ (FGSM) & $\mathcal{A}_{\rm rob}$ (IFGSM$^{20}$) & $\mathcal{A}_{\rm rob}$ (C\&W)\cr
\midrule[0.8pt]
En$_2$ResNet20  & CIFAR10  &  80.34 & 57.23 & 50.06 & 66.47  \cr
WRN-14-2        & CIFAR10  &  78.37 & 52.93  & 48.85 & 60.30  \cr
\bottomrule[1.0pt]
\end{tabular}
\end{threeparttable}
\end{table}

\subsection{Gradient Mask and Comparison with Simple Ensembles}
Besides applying EOT gradient, we further verify that our defense is not due to obfuscated gradient. We use IFGSM$^{20}$ to attack naturally trained (using the same approach as that used in \cite{ResNet}) En$_1$ResNet20, En$_2$ResNet20, and En$_5$ResNet20, and the corresponding accuracies are: $0$\%, $0.02$\%, and $0.03$\%, respectively. All naturally trained EnResNets are easily fooled by IFGSM$^{20}$, thus gradient mask does not play an important role in EnResNets for adversarial defense \cite{Athalye:2018}.
\medskip

Ensemble of models for adversarial defense has been studied in \cite{Strauss:2017}. Here, we show that ensembles of robustly trained ResNets without noise injection cannot boost natural and robust accuracy much. The natural accuracy of jointly (separately) adversarially trained ensemble of two ResNet20 without noise injection is $75.75$\% ($74.96$\%), which does not substantially outperform ResNet20 with a natural accuracy $75.11$\%. The corresponding robust accuracies are $51.11$\% ($51.68$\%), $47.28$\% ($47.86$\%), and $59.73$\% ($59.80$\%), respectively, under the FGSM, IFGSM$^{20}$, and C\&W attacks. These robust accuracies are much inferior to that of En$_2$ResNet20. Furthermore, the ensemble of separately trained robust ResNet20 and robust ResNet44 gives a natural accuracy of $77.92$\%, and robust accuracies are $54.73$\%, $51.47$\%, $61.77$\% under the above three attacks. These results reveal that ensemble adversarially trained ResNets via the Feynman-Kac formalism is much more accurate than standard ensemble in both natural and robust generalizations.

\section{Ensemble of Different ResNets}\label{Section-Generalization}
In previous sections, we proposed and numerically verifies the efficiency of the EnResNet, which can be regarded as an Monte Carlo (MC) approximation to the Feynman-Kac formula that used to solve the convection-diffusion equation. A straightforward extension is to solve the convection-diffusion equation by the multi-level MC \cite{Giles:2018}, which in turn can be simulated by an ensemble of ResNets with different depths. In previous ensembles, we used the same weight for each individual ResNet. However, in the ensemble of different ResNets, we learn the optimal weight for each component. Here, we derive the formula to learn the optimal weights in the cross-entropy loss setting.
\medskip

Suppose we have an ensemble of two ResNets for $n$-class classification with training data $\{\mathbf{x}_i, y_i\}_{i=1}^N$ where $y_i$ is the label of $\mathbf{x}_i$ and $N$ is the number of training data. Let the tensors before the softmax output activation of two ResNet, respectively, be
$$
\tilde{\mathbf{y}}_i = \left(\tilde{y}_i^1, \tilde{y}_i^2 \dots, \tilde{y}_i^n\right),
$$
and 
$$
\hat{\mathbf{y}}_i = \left(\hat{y}_i^1, \hat{y}_i^2, \dots, \hat{y}_i^n\right),
$$
where $i=1, 2, \dots, N$.
\medskip

\noindent The ensemble of these two ResNets gives the following output before the softmax output activation for the $i$-th instance
$$
\mathbf{y}_i = w_1 \tilde{\mathbf{y}}_i + w_2 \hat{\mathbf{y}}_i = \left(\begin{array}{c}
w_1\tilde{y}_i^1 + w_2\hat{y}_i^1\\
w_1\tilde{y}_i^2 + w_2\hat{y}_i^2\\
\vdots\\
w_1\tilde{y}_i^n + w_2\hat{y}_i^n\\
\end{array}\right).
$$
where $w_1$ and $w_2$ are the weights of the two ResNets, where we enforce $w_1+w_2=1$.
Hence, the corresponding log-softmax for the $i$-th instance is
$$
\left(\begin{array}{c}
\log\left( \frac{\exp{(w_1\tilde{y}_i^1 + w_2\hat{y}_i^1)}}{\sum_{j=1}^n \exp{(w_1\tilde{y}_i^j + w_2\hat{y}_i^j)} } \right)\\
\log\left( \frac{\exp{(w_1\tilde{y}_i^2 + w_2\hat{y}_i^2)}}{\sum_{j=1}^n \exp{(w_1\tilde{y}_i^j + w_2\hat{y}_i^j)} } \right)\\
\vdots\\
\log\left( \frac{\exp{(w_1\tilde{y}_i^n + w_2\hat{y}_i^n)}}{\sum_{j=1}^n \exp{(w_1\tilde{y}_i^j + w_2\hat{y}_i^j)} } \right)\\
\end{array}\right),
$$
Let $L$ be the total cross-entropy loss on these $N$ training data, then we have
\begin{equation}\label{w1-update}
\frac{\partial L}{\partial w_1} = - \sum_{i=1}^N \left(\tilde{y}_i^{t_i} - \frac{\sum_{j=1}^n\tilde{y}_i^j \exp{(w_1\tilde{y}_i^j + w_2\hat{y}_i^j)}}{\sum_{j=1}^n \exp{(w_1\tilde{y}_i^j + w_2\hat{y}_i^j)} }\right),
\end{equation}
and
\begin{equation}\label{w2-update}
\frac{\partial L}{\partial w_2} = -\sum_{i=1}^N \left(\hat{y}_i^{t_i} - \frac{\sum_{j=1}^n\hat{y}_i^j \exp{(w_1\tilde{y}_i^j + w_2\hat{y}_i^j)}}{\sum_{j=1}^n \exp{(w_1\tilde{y}_i^j + w_2\hat{y}_i^j)} }\right).
\end{equation}
In implementation, we update these weights once per epoch during the training and normalize the updated weights.
\medskip

To show performance of ensembles of jointly trained different ResNets, we robustly train an ensemble of noise injected ResNet20 and ResNet32 on both CIFAR10 and CIFAR100 benchmarks. As shown in Tables~\ref{EARM-DifferentNets-Cifar10} and \ref{EARM-DifferentNets-Cifar100}, on CIFAR10 the ensemble of jointly trained noise injected ResNet20 and ResNet32 outperforms En$_2$ResNet32 in classifying both clean ($81.46$\% v.s. $81.56$\%) and adversarial images of C\&W attack ($68.41$\% v.s. $68.62$\%).  On CIFAR100, performances of the ensemble of jointly trained noise injected ResNet20 and ResNet32 and En$_2$ResNet32 are comparable.

\begin{table}[tp]
\centering
\fontsize{8.5}{8.5}\selectfont
\begin{threeparttable}
\caption{Natural and robust accuracies of the robustly trained En$_2$ResNet32 and En$_1$ResNet20\&En$_1$ResNet32 on the CIFAR10 dataset. Unit: \%.}\label{EARM-DifferentNets-Cifar10}
\begin{tabular}{ccccc}
\toprule[1.0pt]
Model & dataset & $\mathcal{A}_{\rm nat}$ & $\mathcal{A}_{\rm rob}$ (IFGSM$^{20}$) & $\mathcal{A}_{\rm rob}$ (C\&W)\cr
\midrule[0.8pt]
En$_2$ResNet32                  & CIFAR10  & 81.46    & 52.06 & 68.41  \cr
En$_1$ResNet20\&En$_1$ResNet32  & CIFAR10  & 81.56    & 51.99 & 68.62  \cr
\bottomrule[1.0pt]
\end{tabular}
\end{threeparttable}
\end{table}

\begin{table}[tp]
\centering
\fontsize{8.5}{8.5}\selectfont
\begin{threeparttable}
\caption{Natural and robust accuracies of the robustly trained En$_2$ResNet32 and En$_1$ResNet20\&En$_1$ResNet32 on the CIFAR100 dataset. Unit: \%.}\label{EARM-DifferentNets-Cifar100}
\begin{tabular}{ccccc}
\toprule[1.0pt]
Model & dataset & $\mathcal{A}_{\rm nat}$  & $\mathcal{A}_{\rm rob}$ (IFGSM$^{20}$) & $\mathcal{A}_{\rm rob}$ (C\&W)\cr
\midrule[0.8pt]
En$_2$ResNet32                  & CIFAR100  & 53.14    & 27.27 & 41.50  \cr
En$_1$ResNet20\&En$_1$ResNet32  & CIFAR100  & 53.07    & 27.01 & 42.23  \cr
\bottomrule[1.0pt]
\end{tabular}
\end{threeparttable}
\end{table}

\section{Concluding Remarks}\label{Section-Conclusion}
Motivated by the transport equation modeling of the ResNet and the Feynman-Kac formula, we proposed a novel ensemble algorithm for ResNets. The proposed ensemble algorithm consists of two components: injecting Gaussian noise to each residual mapping of ResNet, and averaging over multiple jointly and robustly trained baseline ResNets. Numerical results on the CIFAR10 and CIFAR100 show that our ensemble algorithm improves both natural and robust generalization of the robustly trained models. Our approach is a complement to many existing adversarial defenses, e.g., regularization based approaches for adversarial training \cite{Zhang:2019-Trades}. It is of interesting to explore the regularization effects in EnResNet.
\medskip

The memory consumption is one of the major bottlenecks in training ultra-large DNNs. Another advantage of our framework is that we can train small models and integrate them during testing. 

\section*{Acknowledgments}
This material is based on research sponsored by the Air Force Research Laboratory under grant numbers FA9550-18-0167, DARPA FA8750-18-2-0066, and MURI FA9550-18-1-0502, the Office of Naval Research under grant number N00014-18-1-2527, the U.S. Department of Energy under grant number DOE SC0013838, the National Science Foundation under grant number DMS-1554564, (STROBE), and by the Simons foundation. Zuoqiang Shi is supported by NSFC 11671005. Bao Wang thanks Farzin Barekat, Hangjie Ji, Jiajun Tong, and Yuming Zhang for stimulating discussions.


\end{document}